\documentclass{article}
\usepackage[preprint]{neurips_alt}
\usepackage{tablefootnote}
\usepackage[utf8]{inputenc}
\usepackage[T1]{fontenc}
\usepackage[hidelinks,pagebackref=true]{hyperref}

\renewcommand*{\backref}[1]{}
\renewcommand*{\backrefalt}[4]{%
    \ifcase #1 {}%
    \or        {\footnotesize[#2]}%
    \else      {\footnotesize[#2]}%
    \fi}

\usepackage{placeins}
\usepackage{url}
\usepackage{booktabs}
\usepackage{xcolor}
\usepackage{amsmath}
\usepackage{amsfonts}
\usepackage{nicefrac}
\usepackage{microtype}
\usepackage{color}
\usepackage{todonotes}
\usepackage{bbm}
\usepackage{enumitem}
\usepackage{bbm}
\usepackage{listings}

\usepackage[ruled]{algorithm2e}

\presetkeys{todonotes}{inline}{}
\presetkeys{todonotes}{color=blue!10}{}

\newcommand{\R}{\mathbbm{R}}
\newcommand{\Ind}{\mathbbm{1}}
\renewcommand{\sectionautorefname}{Section}

\DeclareMathOperator{\E}{E}
\DeclareMathOperator{\Var}{Var}
\DeclareMathOperator{\Cov}{Cov}
\DeclareMathOperator{\SD}{SD}
\DeclareMathOperator*{\argmax}{arg\,max}

\newcommand{\minsize}{\nu}
\newcommand{\0}{\mathbf{0}}
\newcommand{\I}{\mathbbm{I}}

\newcommand{\Samp}{\mathcal{S}}

\newcommand{\X}{\mathcal{X}}

\usepackage{amsthm}
\newtheorem{proposition}{Proposition}
\newtheorem*{lemma*}{Lemma}

\title{\Large
Personalized Assignment to One of Many Treatment Arms
\\ via Regularized and Clustered Joint Assignment Forests
}

\author{
    Rahul Ladhania \\ University of Michigan
    \And Jann Spiess \\ Stanford University
    \AND 
    Lyle Ungar \\ University of Pennsylvania
    \And Wenbo Wu \\ New York University
}

\date{October 31, 2023}

\makeatletter
\renewcommand{\@noticestring}{
\hrule
\vspace{0.5em}
Rahul Ladhania, School of Public Health, University of Michigan, \href{mailto:ladhania@umich.edu}{ladhania@umich.edu}. Jann Spiess, Graduate School of Business, Stanford University, \href{mailto:jspiess@stanford.edu}{jspiess@stanford.edu}. Lyle Ungar, Computer and Information Science, University of Pennsylvania, \href{mailto:ungar@cis.upenn.edu}{ungar@cis.upenn.edu}.
Wenbo Wu, Grossman School of Medicine and Center for Data Science, New York University, \href{mailto:wenbo.wu@med.nyu.edu}{wenbo.wu@med.nyu.edu}.
We thankfully acknowledge support from Schmidt Futures through the ``Advance the Science of Behavior Change through Machine Learning'' project with the \hyperlink{https://bcfg.wharton.upenn.edu}{Behavior Change For Good Initiative} at the University of Pennsylvania.
For comments and discussions, we thank Nathan Kallus, Katy Milkman, Sendhil Mullainathan, and Stefan Wager.}
\makeatother

\usepackage[hang,flushmargin]{footmisc} 

\begin{document}

\begin{titlepage}

\maketitle

\vspace{-2em}
\begin{center}
    October 31, 2023 \\ \: \\ \textit{Comments welcome.}
\end{center}
\vspace{1em}

\begin{abstract}
		We consider learning personalized assignments to one of many treatment arms from a randomized controlled trial.
		Standard methods that estimate heterogeneous treatment effects separately for each arm may perform poorly in this case due to excess variance.
		We instead propose methods that pool information across treatment arms:
	    First, we consider a regularized forest-based assignment algorithm based on greedy recursive partitioning that shrinks effect estimates across arms.
	    Second, we augment our algorithm by a clustering scheme that combines treatment arms with consistently similar outcomes.
	    In a simulation study, we compare the performance of these approaches to predicting arm-wise outcomes separately, and document gains of directly optimizing the treatment assignment with regularization and clustering.
	    In a theoretical model, we illustrate how a high number of treatment arms makes finding the best arm hard, while we can  achieve sizable utility gains from personalization by regularized optimization.
	\end{abstract}

\end{titlepage}

\section{Introduction}

We tackle the problem of learning an assignment policy that maps individual characteristics to one of potentially many treatment arms based on data from a randomized controlled trial. 
We consider a class of regularized forest-based algorithms that directly optimize for the outcome under assignment,
propose a modification that also clusters treatment arms, document the performance of these approaches in a simulation exercise, and discuss extensions. Additionally, we show how having many treatment arms limits the prospects for best-arm identification and the performance of assignments based on separate arm-wise predictions, 
while achieving gains from better assignments remains feasible by optimizing for them directly.

Controlled trials in which treatment has been randomized across arms not only allow for the estimation of average effects, but they also provide an opportunity to learn which treatment works best for whom.
A growing literature brings together methods from machine learning and causal inference to estimate heterogeneous treatment effects with flexible function forms \cite[e.g.][]{athey2016recursive,athey2019generalized,wager2017estimation} and by leveraging such information to optimize the assignment of treatments to individuals \cite[such as][]{athey2020policy,Hitsch2018-bw,Sverdrup2020-ze}.
Other approaches directly optimize the utility of assignment of individuals to treatment \cite[such as][]{Kallus2017-pr,Kitagawa2018-sp} or the probability of assigning individuals to the best arm \cite[e.g.][]{murphy2005generalization,zhou2018sequential}.
Yet many of these procedures focus on the case of a single treatment and control group \cite[with some recent exceptions, such as the multi-arm settings considered in][]{zhou2018sequential,Sverdrup2020-ze,nie2021quasi,zhou2022offline,malearning}.

In a number of clinical and behavior science settings, there might be a large suite of candidate interventions being tested for efficacy.
Examples include psychological theory-informed nudges aimed at promoting gym visits or vaccine uptakes \citep{milkman2021mega, milkman2021megastudy} or anti-depressants being considered for treatment of major depression \citep{ogawa2018efficacy}.
In such cases, existing methods that focus on settings with only a few treatment arms may not adapt well to the many-arms setting.
For example, calculating personalized policies based on separate arm-wise treatment effect estimation may yield excessive variance in estimates and assignments.

In order to estimate personalized assignments with many treatment arms from a randomized trial,
we focus on directly optimizing assignments using data across all arms jointly.
Specifically, we consider a tree-based construction following the ``personalization forest'' proposed by \citet{Kallus2017-pr} that pools information from all treatment arms and directly optimizes for the utility achieved under assignment \citep[similar to the empirical welfare maximization of][]{Kitagawa2018-sp}, in contrast to approaches that select models for the estimation of treatment effects and/or estimate outcomes of different arms separately.
Like \citet{athey2019generalized}, we estimate personalized treatment effects from the combination of trees in an honest way following \citet{athey2016recursive}. We then leverage honest estimates from the training data to obtain an assignment rule.

To achieve better assignments for a large number of treatment arms,
we augment our assignment forest in two ways.
First, we propose a within-leaf regularization scheme that shrinks estimates towards leaf-wise averages.
Second, we cluster treatment arms into groups based on the similarity of their estimated outcomes across units in the training sample.
We then grow a regularized forest based on grouped treatment arms before recovering personalized assignments to one of the original treatment arms in a last step.
In both cases, we pool information across arms, thus reducing the variance relative to estimating all arms separately. 
In its clustering version, we believe that our setup and approach is most similar to recent work in \citet{malearning}, which proposes a supervised clustering approach via adaptive fusion in a parametric treatment-effect model, while we focus on unsupervised clustering in combination with a non-parametric regularized, honest, and jointly estimated random forest.
Our approach also relates to \citet{banerjee2021selecting} that pools treatments in order to select optimal nudges.

We compare the resulting algorithms to natural reference approaches in a simulation study. 
We show that direct optimization and regularization already lead to large improvements relative to benchmark methods that predict arm-wise outcomes separately, and to modest improvements over methods based on estimating separate heterogeneous treatment effects for all treatment arms relative to the control. We also document that adding clustering can improve assignment rules considerably even relative to the best non-clustered alternatives in our comparison group.

Our approach is motivated by the observation that a large number of arms makes it infeasible to consistently find the best arm and renders assignments based on separate arm-wise estimation inefficient.
We make this claim precise in a theoretical illustration with many treatment arms where we show that best-arm identification becomes hard and assignment policies based on separate prediction of outcomes can perform poorly relative to joint assignment.

We introduce our setup in \autoref{sect:setup} and describe the proposed algorithms in \autoref{sect:treeforest}.
\autoref{sect:simulation} describes our simulation experiment and its results.
We lay out our theoretical illustration in \autoref{sect:feasibility}.
In \autoref{sect:extensions}, we discuss extensions to more complex regularization and clustering schemes, as well as to non-experimental data.
We conclude in \autoref{sect:conclusion} by summarizing our findings and discussing important limitations and open questions for future research.

\section{Setup and Goal}
\label{sect:setup}

We consider data from a randomized experiment with $K+1$ treatment arms, and aim to estimate an assignment that maximizes out-of-sample outcomes for the same distribution of potential outcomes.

For treatments $k \in \{0,1,\ldots,K\}$, where we typically identify $k=0$ with the control arm, we let $Y^k \in \R$ be the response of a given unit when assigned to treatment $k \in \{0,1,\ldots,K\}$. We also assume that there are features $X \in \X$ available that are not affected by treatment.
Our goal is to find an assignment $a: \X \rightarrow \{0,1,\ldots,K\}$ such that the expected response
\begin{equation*}
    U(a) = \mathrm{E}[Y^{a(X)}]
\end{equation*}
on new data drawn from the distribution of $(Y^0,\ldots,Y^K,X)$ is maximal, for the given assignment function $a$. (Hence, we assume that the potential outcomes $Y^k$ are in utility units.)

To find an assignment $a$, we assume that we have $n$ iid samples from an experiment available.
In the experimental data, for each observation $i \in \Samp = \{1,\ldots,n\}$ treatment $T_i \in \{0,1,\ldots,K\}$ was randomly assigned  independently of covariates (with for now fixed propensity scores $p^k, k \in \{0,1,\ldots,K\}$), and we observe data
\begin{align*}
    &(Y_i,T_i,X_i)_{i=1}^n,
    &
    Y_i &= Y^{T_i}_i.
\end{align*}
From this data our goal is to estimate an assignment policy $a: \X \rightarrow \{0,1,\ldots,K\}$ that maximizes $U(a)$,
where the optimal (but generally infeasible) assignment policy is given by $a^*(X) = \argmax_a \E[Y^a|X]$.

\section{Regularized Joint Assignment Tree and Forest}
\label{sect:treeforest}

We provide an algorithm that obtains an assignment $\hat{a}: \mathcal{X} \rightarrow \{0,1,\ldots,K\}$ for $\mathcal{X} = \R^d$.
One natural approach to constructing such an algorithm would be to estimate arm-wise outcomes $\E[Y^a|X]$ or treatment effects $\E[Y^a - Y^0|X]$ separately, and then assign an individual with characteristics $X$ to the arm with the highest estimated outcome or treatment effect.
However, such an approach would yield excess variance, since each separate estimation would only use limited data.
In addition, an algorithm optimized for the precise estimation of arm-wise outcomes or arm-wise treatments effects may not be suboptimal for the related, but different goal of finding an assignment that yields high utility.

Instead, we consider an approach that optimizes directly for an optimal assignment across all treatment arms and thereby pools all data.
Specifically, from the training data we obtain joint assignment trees through recursive partitioning, which we then combine into a single joint assignment forest through bagging.
Relative to the construction of the personalized forests from \citet{Kallus2017-pr}, our implementation adapts honest estimation from \citet{athey2019generalized} and specifically targets the challenge of many treatment arms by integrating regularization and clustering in the construction of trees and honest estimation of leaf-wise treatment effects.

\subsection{Regularized Joint Assignment Tree}

For a given training sample of size $n$, we follow \citet{Kallus2017-pr} in fitting a tree by recursively maximizing the regularized empirical analogue of the utility $U(\hat{a})$ for an assignment $\hat{a}$ determined by leaf-wise maximizers. 

\paragraph{Regularized leaf-wise estimation and assignment for a given tree.}

Given a tree that partitions the covariate space $\R^d$ into leaves, we calculate an assignment $a_\ell$ for leaf $\ell$ as the maximizer of the regularized arm-wise within-leaf averages $\hat{Y}^k_\ell$.
Specifically, we write
\begin{align*}
    \overline{Y}^k_{\ell} &= \frac{1}{N^k_{\ell}} \sum_{X_i \in \ell, T_i=k}  Y_i,
    &
    N^k_{\ell} &= \sum_{X_i \in \ell, T_i=k} 1,
\end{align*}
for the arm-wise average outcomes and counts within a leaf in the training data.
Here, $\overline{Y}^k_{\ell}$ is an unbiased estimator of $\mathrm{E}[Y^k|X \in \ell]$.

The arm-wise averages may be noisy, especially for those arms that have only a few observations in a given leaf.
In a departure from the existing literature, we shrink arm-wise average outcomes $\overline{Y}^k_{\ell}$ towards the overall (weighted) average $\overline{Y}_\ell$ of leaf outcomes (which can be motivated by a homoscedastic Normal means model similar to \autoref{sect:feasibility}) and set
\begin{align*}
    \hat{Y}^k_\ell &= \frac{
        N^k_{\ell} \overline{Y}^k_{\ell} + \lambda_1 \overline{Y}_{\ell}
    } {N^k_{\ell} + \lambda_1}
    &
    &\text{where}
    &
    \overline{Y}_{\ell} &= \frac{\sum_{k=0}^K \frac{N^k_{\ell} \overline{Y}^k_{\ell}}{N^k_{\ell} + \lambda_1} }{\sum_{k=0}^K \frac{N^k_{\ell} }{N^k_{\ell} + \lambda_1}},
\end{align*}
and then choose the assignment $a_\ell = \argmax \hat{Y}^k_\ell$.
This regularization scheme reduces the probability that arms with a few draws that are spuriously high are chosen over arms with a high number of draws and a high true average.
The regularization parameter $\lambda_1$ controls the amount of shrinkage; when $\lambda_1 = 0$, $\hat{Y}^k_\ell = \overline{Y}^k_{\ell}$, and we directly maximize the empirical outcome.
We briefly discuss extensions that embrace heteroscedasticity and shrinking towards overall arm-wise averages in \autoref{sect:extensions} below.

\paragraph{Recursive splitting criterion.}
\label{subsubsect:split}
For the leaf-wise assignments $a_\ell$, we recursively split a leaf $\ell_p$ into $\ell_p = \ell_1 \cup \ell_2$ along splits $x_j \leq c$ by maximizing
\begin{equation*}
    \hat{U}_{\ell_1} + \hat{U}_{\ell_2}
\end{equation*}
for $\hat{U}$ one of the estimators
\begin{align*}
    \hat{U}_\ell^N &= N_{\ell} \hat{Y}^{a_{\ell}}_{\ell}
    &
    &\text{ or }
    &
    \hat{U}_\ell^P &= \frac{N^{a_{\ell}}_\ell}{P^{a_{\ell}}} \hat{Y}^{a_{\ell}}_{\ell}
\end{align*}
of the utility achieved by leaf-wise assignment, where
$
    P^k = \frac{N^k}{N} = \frac{\sum_{T_i=k} 1}{N}
$ and $N_\ell = \sum_{X_i \in \ell} 1$.
In deciding on a split, we consider only splits for which:
\begin{enumerate}
    \item there are at least $\minsize$ number of units in the child nodes $\ell_1, \ell_2$;
    \item the increase in utility is at least $\varepsilon \SD(Y)$,
        \begin{equation*}
            \hat{U}_{\ell_1} + \hat{U}_{\ell_2} \geq \hat{U}_{\ell_p} +\varepsilon \SD(Y),
        \end{equation*}
        where $\SD(Y)$ is the empirical standard deviation of the overall outcome variable;
    \item the two child nodes have different optimal treatment assignments, $a_{\ell_1} \neq a_{\ell_2}$.
\end{enumerate}
If no such splits exist, we do not split the leaf $\ell_p$.
This greedy algorithm yields a partition of $\R^d$.

\subsection{Regularized Joint Assignment Forest}
\label{sect:forestaggregation}

We obtain an assignment forest by bagging many trees as in \citet{Kallus2017-pr}, and estimating  honest estimates of the treatment-specific counterfactual outcomes  on the training sample following \citet{wager2018estimation}.

\paragraph{Bagging with treatment-arm randomization.}
Instead of estimating a single tree, we obtain $M$ trees from the training sample by repeatedly drawing a bootstrap sample of size $\lceil\beta n\rceil$, without replacement, for $\beta \in (0,1)$ and repeating the above procedure on the chosen sample.
We stratify each bootstrap sample by treatment arms to ensure that each draw represents the overall fraction of different treatments.
When fitting the trees, we also at every split:
\begin{itemize}
    \item Randomly choose $d'=\delta d$ of the $d$ covariates to consider for the splits, for some $\delta \in (0,1]$;
    \item Randomly choose $K'=\kappa K$ of the $K$ treatment arms to consider for assignment in the child leaves, where we calculate the gain in utility relative to an optimal assignment to one of the chosen arms in the parent leaf, for some $\kappa \in (0,1]$.
\end{itemize}

\paragraph{Honest estimation.}
For every point $x \in \R^d$ and every tree $m \in \{1,\ldots,M\}$,
we follow \citet{athey2019generalized} in obtaining honest estimates of the treatment-specific conditional potential outcome $E[Y^k|X]$, where we estimate expected outcomes only from data that the same tree was not fit on. 
Specifically, denote by $\Samp_m$ the bootstrap sample $m$ was fit on, and write $\ell_m(X) \subseteq \R^d$ for the leaf  that $x$ falls into.
Then we let for all arms $k \in \{0,\ldots,K\}$
\begin{align*}
    \bar{f}^k_m(X) &= \frac{\sum_{i \in H_m^k(X)} Y_i}{n^k_m(X)},
    &
    n^k_m(X) &= |H_m^k(X)|,
    &
    H_m^k(X) &= \{i \in \Samp \setminus \Samp_m;  X_i \in \ell_m(X), T_i = k \}
\end{align*}
and set
\begin{align*}
    \hat{f}^k_m(X) &= \frac{
        n^k_m(X) \bar{f}^k_m(X) + \lambda_2 \bar{f}_m(X)
    } {n^k_m(X) + \lambda_2}
    &
    &\text{where}
    &
    \bar{f}_m(X) &= \frac{\sum_{k=0}^K \frac{n^k_m(X) \bar{f}^k_m(X)}{n^k_m(X) + \lambda_2} }{\sum_{k=0}^K \frac{n^k_m(X) }{n^k_m(X) + \lambda_2}}.
\end{align*}
Here, we allow the shrinkage parameter to differ between the construction of individual trees ($\lambda_1$) and the final, honest estimates of arm-specific outcomes ($\lambda_2$).
This distinction allows us, for example, setting a lower $\lambda_2$ to avoid over-smoothing in constructing these final estimates (by choosing a lower $\lambda_2$), which are averaged over a large number of trees, while also choosing a higher $\lambda_1$  to avoid overfitting in the construction of individual trees.

\paragraph{Aggregation and assignment.}
\label{subsubsect:aggregation}

Given tree-wise honest estimates $\hat{f}^k_m(X)$ at a new sample point $x \in \R^d$, we estimate the conditional potential outcome $E[Y^k|X]$ by the average
\begin{equation}
\label{eqn:armwisepredictions}
    \hat{f}^k(X) = \frac{\sum_{m=1}^M \hat{f}^k_m(X)}{M}.
\end{equation}
We then obtain the assignment
\begin{equation}
    \label{eqn:armassignment}
    \hat{a}(X) = \argmax_k \hat{f}^k(X).
\end{equation}

\paragraph{Tuning parameters.} The parameters $\minsize$ (minimal leaf size), $\lambda_1$ (within-leaf shrinkage when growing the tree), $\lambda_2$ (within-leaf shrinkage when estimating), $\varepsilon$ (minimal gain in objective), $\beta$ (fraction sampled for each tree), $\delta$ (fraction of covariates considered at each split), $\kappa$ (fraction of treatments considered at each split) are the tuning parameters of the assignment forest that control the complexity of the procedure, which can be chosen by cross-validation. The way of calculating $\hat{U}$ ($N$ vs $P$) can be seen as another tuning choice.

\subsection{Reducing Baseline Variation}

Some of the variation in outcomes is common across treatment arms. To the degree that this common variation can be predicted, we can reduce the variance in the evaluation of different treatment arms by subtracting such common variation.
We therefore add a pre-processing step to our algorithm.

\paragraph{Arbitrary residualization.}
\label{subsect:residualization}
Assume we had some fixed function $\bar{f}: \R^d \rightarrow \R$ available. Then ranking between assignments for potential outcomes $\tilde{Y}^k = Y^k - \bar{f}(X)$ are the same as in the original distribution, since
\begin{equation*}
    \tilde{U}(a) = \mathrm{E}[\tilde{Y}^{a(X)}] = \mathrm{E}[Y^{a(X)} - \bar{f}(X)]= \mathrm{E}[Y^{a(X)}] - \mathrm{E}[\bar{f}(X)] = U(a) - \mathrm{E}[\bar{f}(X)],
\end{equation*}
where $\mathrm{E}[\bar{f}(X)]$ does not vary with $a$.
We can therefore estimate an assignment on data $(Y_i - \bar{f}(X_i),T_i,X_i)_{i=1}^n$, where we choose $\bar{f}$ to reduce the variance of the outcome.

\paragraph{Choices of baseline $\bar{f}$.}
We consider three (oracle) choices for a function $\bar{f}: \R^d \rightarrow \R$ to reduce variation in estimating differential assignments:
\begin{enumerate}
    \item The raw average $\bar{f}(X) = \mathrm{E}[Y|X]$. We can estimate this average by regressing $Y$ on $X$ without regard for $W$.
    \item The control baseline $\bar{f}(X) = \mathrm{E}[Y^0|X] = \mathrm{E}[Y|T=0, X]$. We can estimate this average by regression $Y$ on $X$ among those in the control group ($T=0$), if  a designated control group exists. While somewhat arbitrary, we may motivate this choice by a desire to learn first and foremost which units should be assigned to control vs other arms.
    \item The weighted average
    \begin{equation*}
        \bar{f}(X) = \frac{\mathrm{E}\left[Y / (p^{T})^2 \middle|X\right]}{\mathrm{E}[1/(p^{T})^2]}
        =
        \frac{\sum_{k=0}^K \mathrm{E}\left[Y\middle|T=k,X\right] / p^k}{\sum_{k=0}^K 1/p^k}
    \end{equation*}
    that takes into account that outcomes assigned to treatment $T=k$ get weighted by empirical analogues of the inverse propensity score $1/{p^k}$ when constructing leaf-wise averages, generalizing the approach of \citet{Wu:2017wj} (see \autoref{sect:weighting} for details).%
    \end{enumerate}

\paragraph{Implementation.}
Following e.g. \citet{Wager:2016dz,Wu:2017wj}, we fit baseline prediction functions $\widehat{f}_{-i}$ that estimate $\bar{f}$ in the training dataset using cross-fitting to avoid biases from overfitting.
We then run the forest algorithm on the residualized outcomes $Y_i - \widehat{f}_{-i}(X_i)$ where $\widehat{f}_{-i}$ does not use data from observation $i$.
As our main implementation, we solve the weighted prediction problem $$\mathrm{E}\left[
    \left(
        Y - f(X)
    \right)^2 / (p^T)^2
\right]
\rightarrow \min_{f}$$
\citep[which follows][]{spiessoptimal}
using a random forest, which implements the third option (see \autoref{sect:weighting}).
We do not residualize on the held-out data, since residualization shifts the absolute policy value of the assignment, which may be of interest. Our algorithm is summarized as \autoref{alg:dof}.

\begin{algorithm}
\caption{Regularized Joint Assignment Forest with Clustering Option}
\label{alg:dof}
\begin{itemize}[label={}, leftmargin=0.5cm]
\item For a given training sample $S$ of size $n$:
\item \textbf{1. Pre-processing}
\vspace{-.2cm}
    \begin{itemize}[label={}, leftmargin=0.5cm]
    \item Reduce baseline variation in the data by residualization (Section 3.3)
    \end{itemize}
\item \textbf{2. Regularized Joint Assignment Forest}
\vspace{-.2cm}
    \begin{itemize}[label={}, leftmargin=0.5cm]
    \item For every tree $m \in \{1, . . . , M \}$, we perform bagging with treatment-arm and covariate randomization
    \item \textbf{Regularized Joint Assignment Tree}
        \begin{itemize}[label={}, leftmargin=0.3cm]
        \item (a) On one split of the training data, perform regularized leaf-wise estimation and assignment, with recursive splitting criterion specified in Section 3.1
        \item (b) Estimate regularized “honest” outcome averages on the other split
        \end{itemize}
    \item Aggregate estimates from each of the $M$ trees and learn assignment rule as in Section 3.2
    \end{itemize}
\item \textbf{3. Clustering of Treatment Arms}
\vspace{-.2cm}
    \begin{itemize}[label={}, leftmargin=0.5cm]
    \item If clustering, perform Step 2 using an $F$-fold approach to fit assignment forests and estimate $K + 1$ outcomes for each of the $n$ units in training sample $S$ (Section 3.4)
    \item (a) Cluster $K$ vectors corresponding to non-control arms $k \geq 1$ into $M$ groups to obtain $M + 1$ “arms”
    \item (b) Repeat Step 2 on the full training data $S$ with $M + 1$ arms, and obtain regularized estimates on the original $K + 1$ arms 
    \end{itemize}
\end{itemize}
\end{algorithm}

\subsection{Clustering of Treatment Arms}
\label{sect:clustering}

While simple regularization within arms reduces variation in model construction due to random outliers, it is not able to capture systematic similarity between specific arms.
In this section, we therefore propose a simple clustering scheme that partitions the $K+1$ arms into $M+1$ groups $G$, where $\bigcup_{g \in G} g = \{0,\ldots,K\}$.
Like \citet{Bonhomme2015-tf} and \citet{Bonhomme2017-pd} in their study of unobserved heterogeneity, we employ a $k$-means algorithm for clustering the $K+1$ arms into $M+1$ groups.
Specifically, we expand the assignment forest from the previous section as follows:
\begin{enumerate}
    \item 
    Randomly split the training data into $F$ folds.
    For each fold $f \in \{1,2, \dots, F\}$, we denote the units in $f$-th fold as the on-fold sample $\Samp_f \subseteq \Samp$ and the rest of the training data units as the off-fold sample $\Samp_{- f} = \Samp \setminus \Samp_f$.
    \begin{enumerate}
        \item Fit the assignment-forest algorithm from \autoref{sect:treeforest} on the off-fold sample $\Samp_{-f}$ to obtain prediction functions $\hat{f}^k_f:\X \rightarrow \R$ of potential outcomes as in \eqref{eqn:armwisepredictions}.
        \item For each observation $i \in \Samp_f$ and arm $k \in \{0,\ldots,K\}$ in the on-fold sample, obtain fitted values $\hat{y}^k_i = \hat{f}^k_f(x_i)$.
    \end{enumerate}
    This gives us $K+1$ predictions $\hat{y}^k_i$ for each of the $n$ units in the training sample $\Samp$.
    \label{itm:preprocess}
    \item Either cluster the $K$ vectors $\hat{y}^k = (\hat{y}^k_i)_{i \in \Samp}$ corresponding to non-control arms $k \geq 1$ into $M$ groups, or
    cluster all of the $K+1$ vectors $\hat{y}^k = (\hat{y}^k_i)_{i \in \Samp}$ corresponding to arm-wise predictions into $M+1$ groups; in both cases, we obtain $M+1$ arms, where we retain the original control arm in the first case.
    \label{itm:cluster}
    \item Repeat the assignment-forest algorithm on the full training data $\Samp$ with $M+1$ arms (where data from the original arms are combined by groups) to obtain an ensemble of trees.
    \label{itm:construct}
    
    \item Obtain final predictions and assignments as in \autoref{sect:forestaggregation}, where we now go back to estimating regularized averages separately by the original treatment arms $k \in \{0,\ldots,K\}$ and obtain a corresponding assignment.
    \label{itm:assign}
\end{enumerate}

Our proposed algorithm thus uses a clustering of treatment arms when constructing the assignment trees in Step~\ref{itm:construct}, but still provides arm-specific estimates from the resulting forest in Step~\ref{itm:assign}.

We believe that our approach is most similar to the recent clustering approach of \citet{malearning}, which operates within the same framework, but considers a semi-parametric model with parametric treatment effects.
Unlike our unsupervised clustering scheme, \citet{malearning} leverages a fusion penalty term to obtain supervised clustering.

\section{Simulation Study} 
\label{sect:simulation}
    
We conduct a simulation study to examine the performance of the algorithms described above. 

\subsection{Simulation Setup}
\label{subsect:simulationdetails}

For $X \sim \mathcal{N}(\0_d,\I_d), \varepsilon \sim \mathcal{N}(0,1)$, $T \sim \mathcal{U}(\{0,1,\ldots,K\})$ with $d=3$ we generate outcomes by
\begin{align*}
    Y &= 
        10 + 20 \cdot \Ind_{X_1 > 0} - 20  \cdot \Ind_{X_2 > 0}
        - 40 \cdot \Ind_{X_1, X_2 > 0}
    \\
    &
    \phantom{=} + \gamma \cdot
    (2  \cdot \Ind_{X_3 > 0} - 1) \frac{2 \cdot T - K - 1}{K-1} \Ind_{T>0}
    \\
    &
    \phantom{=}
    - 10 \cdot X_1^2 \Ind_{T=0}
    + \sigma \cdot \varepsilon.
\end{align*}
In this $K+1$-treatment-arm setting (with choice of two parameters, the strength of treatment effects $\gamma$ and residual noise $\sigma$), the optimal treatment choices are $T=K$ (for $X_3 > 0$) and $T=1$ (otherwise), realizing an average outcome of $\gamma$. Average outcome for assigning to control is $-10$, for assigning to the global best is $0$, and for assigning randomly among one of the $K$ non-control treatments is also $0$. We run simulations for treatment effect strengths $\gamma \in \{10,20\}$ and noise level $\sigma \in \{10,20\}$  with $K \in \{9,29,49,99\}$ arms on samples of size $n = 5{,}000$.

We choose a simple setting with joint baseline and treatment-effect structure to clearly highlight the advantages of the joint assignment approach. Note, however, that treatment arms are not clustered and all have different effects on the outcome.

\subsection{Tuning Parameters}
For our three-covariate simulation setup, we tune the following parameters:  $\minsize \in \{3,5\}$ (minimal leaf size), $\lambda_1 \in \{0,0.5,1\}$ (within-leaf shrinkage when growing the tree), $\lambda_2 \in \{0,0.5,1\}$ (within-leaf shrinkage when estimating), $\varepsilon \in \{0.5,1,2\}$ (minimal gain in objective),  $\kappa \in \{0.5, 0.8, 1\}$ (fraction of treatments considered at each split), and choose the combination of these parameters for each arm-``noise'' setting by cross-validation. We use a test set of size 10,000 for arriving at our final estimates.  We note that it may make sense to use different regularization parameters across Steps~\ref{itm:preprocess}, \ref{itm:construct}, and \ref{itm:assign} in \autoref{sect:clustering}.
In our simulations, we used the same regularization parameters ($\lambda_1$ and $\lambda_2$ before and after clustering) and found that works well.

\subsection{Comparison Methods}

We compare the performance of our regularized and clustered approach to two natural approaches to determining treatment assignments.
First, we fit $K+1$ random forests separately for each control and treatment arm $k \in \{0,1,\dots, K\}$, and predict $\hat{Y}_k = \hat{f}_k(X)$ for each observation under consideration.
The assignment rule is obtained as
$\hat{a}(X) = \argmax _{k \in \{0,1,\dots, K\}} \hat{f}_{k}(X)$.
Second, we compare the performance of our approach to a multi-arm causal forest \citep{athey2019generalized,nie2021quasi} that jointly estimates treatment effects $\hat{\tau}_k(X)$ for all $K$ treatment arms $k \in \{1,\ldots,K\}$.
As an assignment rule for the multi-arm causal forest, we use
$\hat{a}(X) = \argmax_{k \in \{0,1,\dots, K\}} \hat{\tau}_{k}(X)$,
where we set $\hat{\tau}_{0}(X) = 0$.

\subsection{Simulation Results}

\autoref{fig:50treatmentsvaluefunc} presents results for the optimal value function  from 500 simulations for a 50-treatment arm design for the following three settings: a. $\gamma = 10, \sigma = 10$ (hereby referred to as regular setting), b. $\gamma = 20, \sigma = 10$ (''low noise'' setting), and c. $\gamma = 10, \sigma = 20$ (``high noise setting'').  \autoref{fig:highnoisevaluefunc} presents results for the optimal value function from 500 simulations of the ''high noise'' ($\gamma = 10, \sigma = 20$) setting for 10, 30, 50, and 100 treatment arms. We present corresponding results of the assignment rate in Figures \ref{fig:50treatmentsassignmentrate} and \ref{fig:highnoiseassignmentrate} in the Appendix.  
Here, `Oracle Optimal Assignment' refers to following the optimal assignment rule, based on the ground truth in the simulation. `Random Assignment' refers to assigning units in every simulation randomly across treatment arms; `Global Best Assignment' refers to assigning units in every simulation the on-average best performing treatment. 

We note that our joint assignment approaches (un-clustered and clustered)  outperform the separate random forests approach in recovering the oracle outcome under optimal assignment for all $\gamma,\sigma$ settings, across all considered treatment arms (10, 30, 50, 100). The clustered DOF approach outperforms the un-clustered approach in recovering the oracle optimal outcome across all settings, making a strong case for treatment arm clustering under many arms. Compared to the multi-arm causal forest approach, our clustered approach outperforms it in recovering the optimal outcome across the ``regular'' and ``high noise'' settings, while  delivering similar performance in a ``low noise'' setting (Figure \ref{fig:50treatmentsvaluefunc}). The resilience of our joint assignment approach compared to the multi-arm causal forest is particularly prominent in `high noise' settings (Figure \ref{fig:highnoisevaluefunc}).

Figures \ref{fig:50treatmentsassignmentrate} and \ref{fig:highnoiseassignmentrate} in the Appendix demonstrate the performance of our clustered and unclustered approaches against the separate forest and multi-arm causal forest on the average (across 500 simulations) of the successful assignment rate to the oracle treatment assignment rule. In settings with fewer arms, we find our approach to perform well (Figure \ref{fig:highnoiseassignmentrate}), but as the number of arms increases, both our unclustered and clustered joint assignment forest approaches perform worse than the ``global best assignment'' rule. While the algorithms perform better in ``low noise'' settings (for large arm settings), in ``high noise'' ones, it approaches assignment rates akin to random assignment \ref{fig:50treatmentsassignmentrate}). 

\begin{figure}[hbt!]
    \centering
    \includegraphics[width=\textwidth]{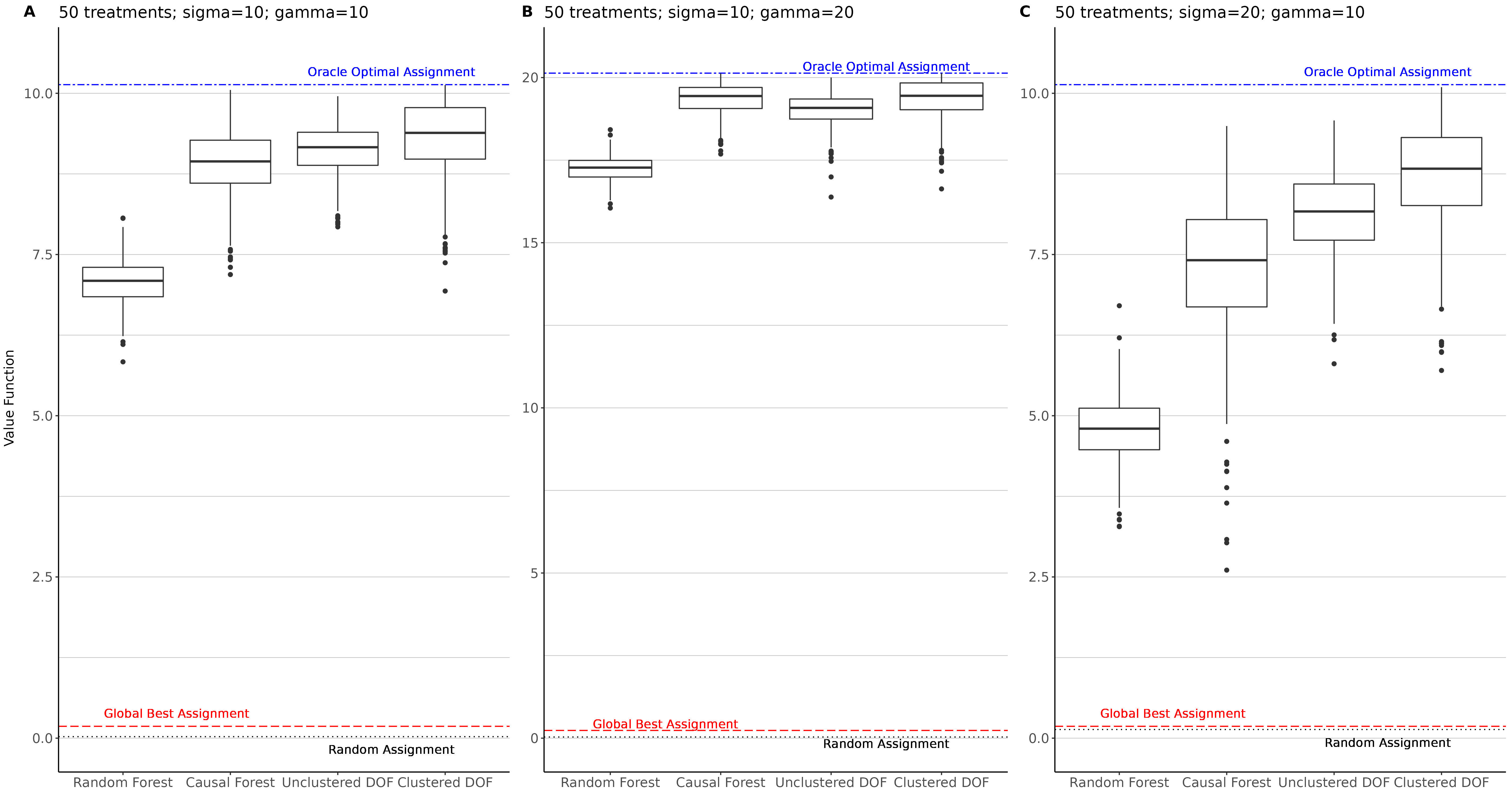}
    \caption{Boxplots of 500 simulations comparing the average out-of-sample outcome of the units under assignment rules learned from (l to r):  separate random forests, multi-arm causal forest, unclustered regularized joint assignment forest (DOF), and clustered regularized joint assignment forest for 50 treatment arms in A. $\gamma = 10, \sigma = 10$ (regular setting), B. $\gamma = 20, \sigma = 10$ (''low noise'' setting), C. $\gamma = 10, \sigma = 20$ (``high noise'' setting)}\label{fig:50treatmentsvaluefunc}
\end{figure}

\begin{figure}[hbt!]
    \centering
    \includegraphics[width=\textwidth]{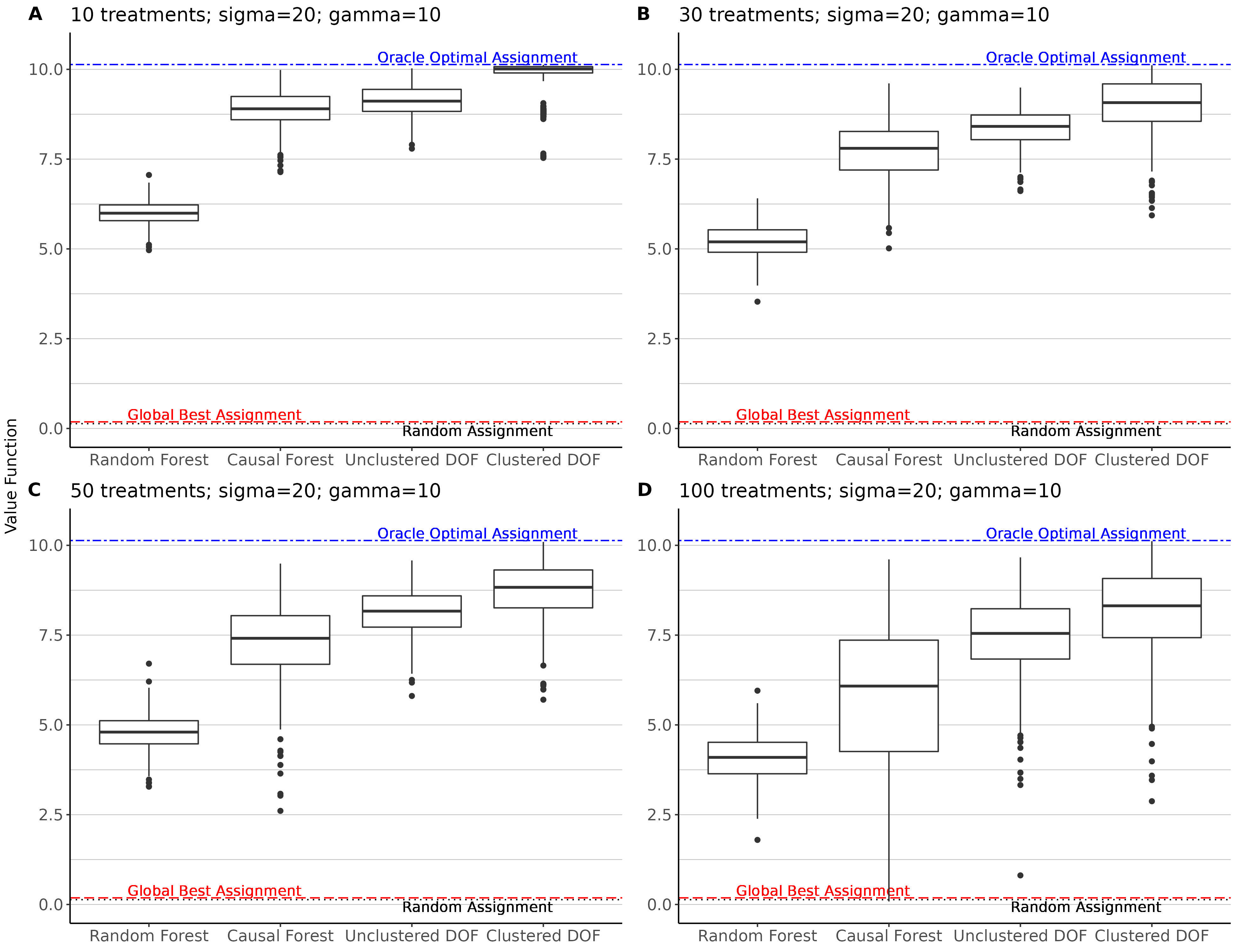}
    \caption{Boxplots of 500 simulations comparing the average out-of-sample outcome of the units under assignment rules learned from (l to r):  separate random forests, multi-arm causal forest, unclustered regularized joint assignment forest, and clustered regularized joint assignment forest in a  ''high noise'' setting ($\sigma = 20, \gamma = 10$) for A. 10, B. 30, C. 50, and D. 100 treatment arms}
    \label{fig:highnoisevaluefunc}
\end{figure} 

\FloatBarrier

\section{Illustration of Challenges in Assignments to One of Many Arms}%
\label{sect:feasibility}

One motivation for our regularized joint optimization approach is that finding optimal treatment assignments becomes hard and separate estimation inefficient when there are many treatment arms.
In this section, we theoretically illustrate these features.

Specifically, we consider different goals and procedures when assigning treatments in a small example.
One natural goal when assigning treatments is maximize the probability of the chosen arm for a randomly chosen unit being the best arm for that unit.
In \autoref{sec:bestarm}, we show in a high-dimensional Normal example that this goal of best-arm identification becomes hard when the number of arms increases, while assigning to maximize utility still can yield non-trivial solutions over the benchmark of random assignment.
We therefore focus instead on the goal of optimizing for a treatment rule with high average outcomes.

A natural procedure to achieving good treatment assignments is to predict each treatment arm's outcome separately and then assign a given unit to the treatment with the highest predicted outcome.
In \autoref{sec:separate},
we compare this method of assigning by estimating the outcomes of different treatment arms separately to optimizing directly for an optimal assignment in the same limiting regime, and show that except for edge cases the former performs strictly worse than the latter even in the limit.
Together, the illustrations in this section motivate our focus in direct utility maximization over best-arm identification and the estimation of separate arms.

\subsection{Simple High-Dimensional Normal Model}

Comparing how hard goals are to achieve and how to best achieve them faces two related hurdles: First, procedures may not only differ in their target loss function, but also in the function class, regularization, and optimization method; and second, performance depends on the true distribution and there may not be a universally best solution, even holding e.g. the function class fixed.
We overcome these challenges by a standard approach from statistical decision theory \citep{Wald1950-qc}:
we consider a distribution over the true state of the world, and compare the average performance for optimal (Bayes) solutions to each of the problems.
This way we can abstract from the specific algorithm employed for each purpose and focus instead of the different optimization goals.

For tractability, we consider a simple homoscedastic baseline model with $\X = \{1,\ldots,N\}$ and
\begin{align*}
    Y | T=k, X=j &\sim \mathcal{N}(\mu^k_j,\sigma^2).
\end{align*}
We assume that the means $\mu^k_j$ are themselves jointly distributed according to a multivariate Normal prior that is invariant to permutations in treatment arms and covariates identities, capturing the idea that these are ex-ante indistinguishable.
This allows us to write
$\mu^k_j = \alpha + \beta_j + \gamma^k + \delta^k_j$ with
\begin{align*}
    \alpha &\sim \mathcal{N}(0,a),
    &
    \beta_j &\sim \mathcal{N}(0,b),
    &
    \gamma^k &\sim \mathcal{N}(0,c),
    &
    \delta^k_j &\sim \mathcal{N}(0,d),
\end{align*}
where all draws are independent.
We also assume that all cells $T=k, X=j$ have the same sample size $\frac{n}{(K+1) N}$.
We study this general model in \autoref{sect:illustration}; in this section, we assume throughout that $a,b,c,d, \sigma^2 > 0$.

For analyzing the case of many treatment arms in a high-dimensional setting, we consider the case where
\begin{align}
    n &\rightarrow \infty,
    &
    N & \rightarrow \infty,
    &
    \frac{n}{(K+1) N} \rightarrow m < \infty,
\end{align}
and both the case of fixed (but potentially high) $K$ and the case where $K \rightarrow \infty$.
This approximation regime represents a case where the number of covariate cells and treatment arms increases fast enough relative to the sample size for the estimation of the cell-wise means $\mu_j^k$ to be hard even in the limit.
This is motivated by cases such as the mega-study in our application, where some of the many treatment arms have only a few hundred observations across all covariate values.

\subsection{Best-Arm Assignment vs Utility Maximization}
\label{sec:bestarm}

We compare two different goals when learning an assignment $\hat{a}: \X \rightarrow \R$:
\begin{enumerate}
    \item Best-arm assignment: maximize $\mathrm{P}(a(X) = \argmax_k \mathrm{E}[Y|T=k,X])$;
    \item Utility maximization: maximize $\mathrm{E}[Y^{a(X)}]$.
\end{enumerate}
In our homoscedastic setting, the average outcomes for both goals are maximized by an assignment rule $\hat{a}$ that picks, for each cell, the arm with the highest posterior expectation.
In \autoref{sect:illustration} we compare this assignment to the infeasible oracle $a^*(j) = \argmax_k \mu_j^k$ (as an upper bound for performance) and the assignment $\underline{a}$ that assigns treatments randomly (as a lower bound for performance),
and show:

\begin{proposition}[Limits of many-treatment best-arm identification]
    As $K \rightarrow \infty$,
    $
        \mathrm{P}(\hat{a}(X) = \argmax_k \mu_X^k) \rightarrow 0.
    $
    At the same time,
    \[
        \frac{\mathrm{E}[Y^{\hat{a}(X)}] - \mathrm{E}[Y^{\underline{a}(X)}]}{\mathrm{E}[Y^{a^*(X)}]- \mathrm{E}[Y^{\underline{a}(X)}]}
        \rightarrow \text{const.} > 0,
    \]
    where the distribution is over the prior and the data.
\end{proposition}

With many treatment arms, finding the best arm hence becomes hard (even with an optimal procedure), while recovering a sizable fraction of the utility gain from personalization remains feasible.
The reason is simple: with more and more arms, we may make selection mistakes in picking a similar arms with slightly lower mean; but that arm likely still has a higher-than-average mean.
Indeed, in \autoref{sect:illustration} we show that the fraction of optimal utility recovered does not change with $K$ in our asymptotic approximation.

\subsection{Separate Prediction vs Joint Assignment}
\label{sec:separate}

We now focus maximizing utility $\mathrm{E}[Y^{a(X)}]$, and consider two natural methods of achieving this goal:
\begin{enumerate}
    \item Arm-wise prediction: Learn separate predictions $\hat{f}^k: \X \rightarrow \R$ from $(Y_i,X_i)$ with $T_i=k$ that minimize $\mathrm{E}[(\hat{f}^k(X) - Y)|T=k]$, determine assignment by $\tilde{a}(X) = \argmax_k \hat{f}^k(X)$;
    \item Direct utility maximization: Learn an assignment $\hat{a}(X)$ that directly maximizes $\mathrm{E}[Y^{a(X)}]$.
\end{enumerate}
While the former strategy is compelling in practice since it can be readily implemented by using separate prediction  algorithms across treatment arms,
we show in \autoref{sect:illustration} that this strategy is suboptimal in our high-dimensional asymptotic approximation:
\begin{proposition}[Limits of arm-wise prediction]
    In our limiting framework,
    \begin{equation*}
        \frac{\mathrm{E}[Y^{\tilde{a}(X)}] - \mathrm{E}[Y^{\underline{a}(X)}]}{\mathrm{E}[Y^{\hat{a}(X)}]- \mathrm{E}[Y^{\underline{a}(X)}]}
        \rightarrow \text{const.} < 1,
    \end{equation*}
    where the distribution is over the prior and the data.
\end{proposition}
When separate prediction algorithms only use data from a single arm, they may misattribute variation in the baseline that is common to all treatment arms with a given covariate value, and therefore overfit to individual cell outcomes.
In a high-dimensional limit, this under-performance does not go away.

\section{Extensions}
\label{sect:extensions}

In this section, we discuss extensions to the baseline algorithm discussed in \autoref{sect:treeforest}.

\subsection{Propensity Scores and Weighting}

The above algorithm assumes that treatments are assigned randomly with constant propensity score. But we can easily modify the algorithm for a known propensity score $p^k(X)$.
If we also care about weighted outcomes
\begin{equation*}
    \frac{E[v(X) \: Y^{a(X)}]}{E[v(X)]}
\end{equation*}
for some given positive weights $v(X)$,
then two natural (unregularized) estimators of leaf-wise utility are
\begin{align*}
    \hat{U}_\ell^N &= \frac{\sum_{i \in \ell} v(X_i)}{\sum_{i \in \ell, T_i = a_\ell} \frac{v(X_i)}{p^{a_\ell}(X_i)}} \sum_{i \in \ell, T_i = a_\ell} v(X_i) \frac{Y_i}{p^{a_\ell}(X_i)}
    &
    &\text{ and }
    &
    \hat{U}_\ell^P &= \sum_{i \in \ell, T_i = a_\ell} v(X_i) \frac{Y_i}{p^{a_\ell}(X_i)}.
\end{align*}
The latter estimator is unbiased for the utility of assigning the leaf accordingly by standard inverse-propensity weighting, yielding an overall unbiased estimate of the associated policy as in \citet{Hitsch2018-bw}.
We obtain the criteria in \autoref{sect:treeforest} from $v \equiv 1$, $p^k \equiv P^k$.

\subsection{Regularization and Shrinkage Across Leaves}

In \autoref{sect:treeforest}, we estimate arm-wise average outcomes within leaves using regularization to avoid noise when there are only a few observations in a given arm.
The regularization scheme implicitly assumes a homoscedastic Normal-means model with a Normal prior on leaf-specific arm-wise averages and an uninformative hyperprior (corresponding to an empirical-Bayes strategy).
The tuning parameter $\lambda$ corresponds to the ratio of the unit-specific variance to the variance of the Normal prior.
As a more complete treatment, we can also include shrinkage towards the overall arm-wise average across leaves.
Further, we can estimate arm-wise variances of outcomes to refine shrinkage.

\subsection{Hierarchical Bayesian Modelling}

In \autoref{sect:clustering}, we consider an ad-hoc $k$-means clustering scheme. We could instead consider a Bayesian model similar to \autoref{sect:feasibility} that assumes that some arms are more similar to each other than others.
Such a model would generalize the shrinkage scheme in \autoref{sect:treeforest} by shrinking arms more towards those in the same group or with higher similarity.
We could then estimate such a grouped shrinkage scheme with a $k$-means algorithm as in  \autoref{sect:clustering} or with an estimation of arm similarity corresponding to covariances of arm-wise means in the Bayesian model.

\subsection{Iterative Clustering}

Our current cluster scheme starts by clustering treatment arms into groups based on a simple, non-clustered run of the assignment-forest algorithm.
Alternatively, we could iterate the prediction and clustering steps to refine the clustering, either starting with a non-clustered or a randomly clustered assignment.
In addition, when clustering we could take into account whether an arm is likely to be chosen for treatment, and ensure that clusters capture similarity mainly for those observations-specific arms that are likely to affect assignment.
Finally, we could maintain the clusters for assignment, assuming that within-group assignments are random.

\section{Conclusion}
\label{sect:conclusion}

In this article, we consider learning treatment assignments from experimental data with many treatment arms. We demonstrate the limits of estimating optimal treatment arms and recovering effective treatment assignment policies from separate arm-wise outcome predictions or treatment-effect estimates.
As a feasible alternative, we provide a regularized tree-based algorithm that directly optimizes for treatment assignment, clusters treatment arms, and document its properties in a simulation study.

Our current analysis remains limited to experiments with known probabilities of assignment to different arms. When treatment assignment is endogenous and propensity scores are unknown, then their estimation poses additional challenges that are beyond the scope of this article. Similarly, we focus on existing experimental data, and do not consider optimal experimentation of the dynamic allocation to treatment arms.

The use of targeting rules, like those obtained from our algorithm, has the potential to improve utility through better allocation, but also comes with substantial risks when data is biased or personalization may reinforce or increase inequities. Adding fairness and equity constraints to the resulting treatment-assignment rules can be an important future extension.

\newpage
\bibliographystyle{apalikefull}
\bibliography{references}

\clearpage
\appendix
\section*{Appendix}
\counterwithin{figure}{section}
\label{appendix}
\renewcommand{\sectionautorefname}{Appendix}

\section{Residualization by Weighted Baseline}
\label{sect:weighting}

The weighted average
\begin{equation*}
    \bar{f}(X) = \frac{\mathrm{E}\left[Y / (p^{T})^2 \middle|X\right]}{\mathrm{E}[1/(p^{T})^2]}
    =
    \frac{\sum_{k=0}^K \mathrm{E}\left[Y\middle|T=k,X\right] / p^k}{\sum_{k=0}^K 1/p^k},
\end{equation*}
which takes into account that outcomes assigned to treatment $T=k$ get weighted by empirical analogues of the inverse propensity score $1/{p^k}$ when constructing leaf-wise averages,
minimizes the average residual variance
\begin{align*}
    & \mathrm{E}\left[\Var\left( (Y - \bar{f}(X)) / p^T |T \right)\right]
    \\
    &=
    \mathrm{E}\left[\left(
        \Var(Y|T) - 2 \Cov(Y,\bar{f}(X)|T) + \Var(\bar{f}(X)|T)
        \right) / (p^T)^2
    \right]
    \\
    &=
    \mathrm{E}\left[
        \Var(Y|T) / (p^T)^2
        - 2 \Cov(Y / (p^T)^2,\bar{f}(X)|T) +
        \Var(\bar{f}(X)) / (p^T)^2
    \right]
    \\
    &=
    \text{const.}
    +
    \mathrm{E}[1/ (p^T)^2]
    \left(
    - 2 \mathrm{E}\left[\Cov\left(\frac{Y / (p^T)^2}{\mathrm{E}[1/ (p^T)^2]},\bar{f}(X)\middle|T\right)\right]
    + \Var(\bar{f}(X)) \right) 
    \\
    &=
    \text{const.}
    + \text{const.} \cdot
    \left(  - 2 \Cov\left(\frac{Y / (p^T)^2}{\mathrm{E}[1/ (p^T)^2]},\bar{f}(X)\right)
    + \Var(\bar{f}(X)) \right)
    \\
    &=
    \text{const.}
    + \text{const.} \cdot
    \Var\left(\frac{Y / (p^T)^2}{\mathrm{E}[1/ (p^T)^2]} - \bar{f}(X)\right).
\end{align*}
This generalizes the adjustment by $p^1  \mathrm{E}[Y^0|X] + p^0  \mathrm{E}[Y^1|X]$ for the case $K=1$ in \citet{Wu:2017wj}.
Generalizing a result in \citet{spiessoptimal},
$\bar{f}$ solves the weighted prediction problem
$
    \mathrm{E}\left[
        v^T
        \left(
            Y - f(X)
        \right)^2
    \right]
    \rightarrow \min_f
$
with weights $v^T = 1 / (p^T)^2$.

\newpage

\section{Treatment Assignment Rates in the Simulation}
\begin{figure}[hbt!]
    \centering
    \includegraphics[width=.9\textwidth]{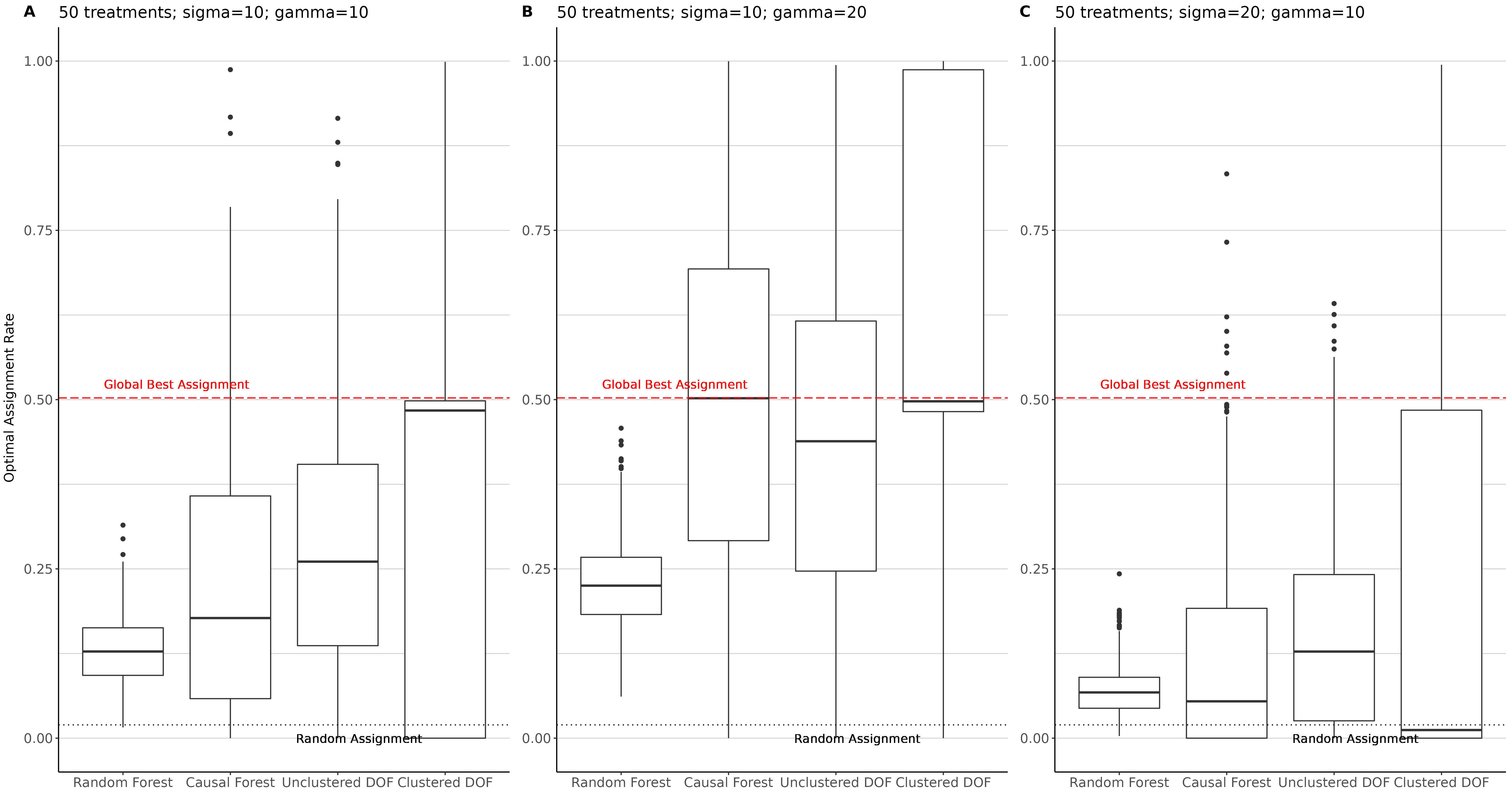}
    \caption{Boxplots of 500 simulations comparing the average assignment rate of the units under assignment rules learned from (l to r):  separate random forests, multi-arm causal forest, unclustered regularized direct optimization forest, and clustered regularized direct optimization forest for 50 treatment arms in A. $\gamma = 10, \sigma = 10$ (regular setting), B. $\gamma = 20, \sigma = 10$ (''low noise'' setting), C. $\gamma = 10, \sigma = 20$ (``high noise'' setting).}
    \label{fig:50treatmentsassignmentrate}
\end{figure}

\begin{figure}[hbt!]
    \centering
    \includegraphics[width=.9\textwidth]{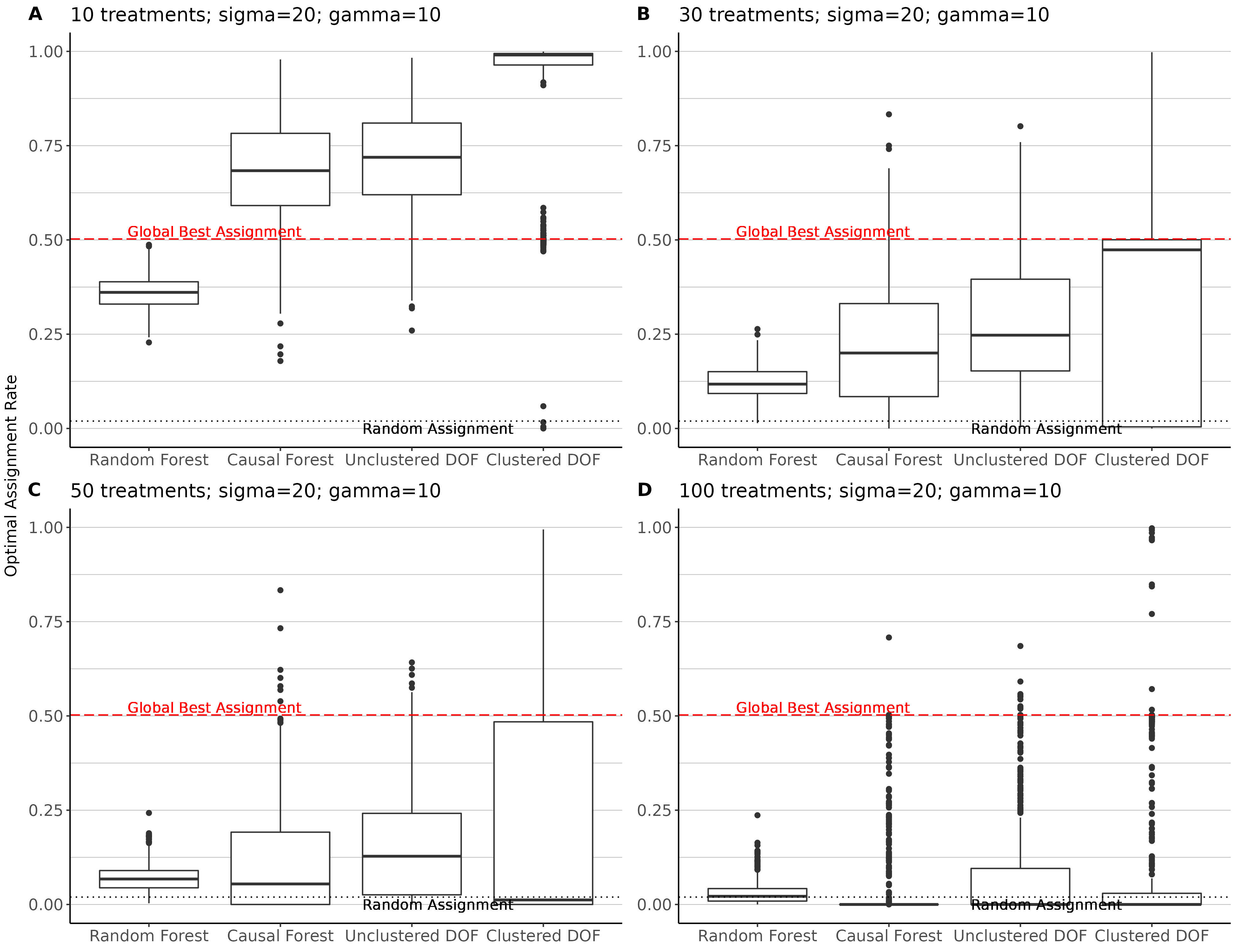}
    \caption{Boxplots of 500 simulations comparing the average average assignment rate of the units under assignment rules learned from (l to r):  separate random forests, multi-arm causal forest, unclustered regularized direct optimization forest, and clustered regularized direct optimization forest in a  ''high noise'' setting ($\sigma = 20, \gamma = 10$ ) for A. 10, B. 30, C. 50, and D. 100 treatment arms.}
    \label{fig:highnoiseassignmentrate}
\end{figure}
\newpage 
\section{Normal Model Illustration}
\label{sect:illustration}

\newcommand{\1}{\mathbf{1}}
\renewcommand{\O}{\mathbbm{O}}

Here we work out the simple Normal model used for illustration in
Section~\ref{sect:feasibility}.

We consider a homoscedastic baseline model with $\X = \{1,\ldots,N\}$ and
\begin{align*}
    Y | T=k, X=j &\sim \mathcal{N}(\mu^k_j,\sigma^2).
\end{align*}
We assume that the $\mu^k_j$ are themselves distributed according to a multivariate Normal prior that is invariant to permutations in treatment arms and covariates identities, capturing the idea that these are ex-ante indistinguishable.
These restrictions imply that we can write
\begin{equation*}
    \mu^k_j = \alpha + \beta_j + \gamma^k + \delta^k_j
\end{equation*}
with independent priors
\begin{align*}
    \alpha &\sim \mathcal{N}(0,a),
    &
    \beta_j &\sim \mathcal{N}(0,b),
    &
    \gamma^k &\sim \mathcal{N}(0,c),
    &
    \delta^k_j &\sim \mathcal{N}(0,d),
\end{align*}
where the zero mean assumption is  for our analysis.
For simplicity, we further assume that all cells $T_i = j, X_i = k$ are of equal size (hence, of size $m_n = n / ((K_n+1) N_n)$).

We now focus on cell $j \in \X$.
By linearity and exchangeability,
the posterior for $\mu_j$ only depends on the data through the vectors
\begin{align*}
    &\overline{Y}_j,
    & 
    &\overline{Y}_{-j}
\end{align*}
of averages of cell $j$ and of outcomes in all other cells.
Let $q =\1_{K+1} / \sqrt{K+1}$ and $Q \in \R^{(k + 1) \times k}$ be such that $(Q,q)$ orthonormal. Only $Q' \mu_j$ is relevant for rankings between treatment arms,
and
\begin{align*}
    Q' \mu_j &= Q'\gamma + Q'\delta_j,
    &
    &Q'\gamma \sim \mathcal{N}(\0,\I c),
    Q'\delta_j \sim \mathcal{N}(\0,\I d),
    \\
    q' \mu_j &= q'\1 \alpha + q'\1 \beta_j + q'\gamma + q'\delta_j,
    &  
    &\alpha \sim \mathcal{N}(0,a),
    \beta_j \sim \mathcal{N}(0,b),
    q'\gamma \sim \mathcal{N}(0,c),
    q'\delta \sim \mathcal{N}(0,d),
\end{align*}
where all distributions are independent for a given $j$.
For the data (integrating over the prior over $\delta_{j'}, j'\neq j$),
\begin{align*}
    Q'\overline{Y}_j &\sim \mathcal{N}\left(Q'\gamma + Q'\delta_j, \I \frac{\sigma^2}{m_n} \right)
    &
    Q'\overline{Y}_{-j} &\sim \mathcal{N}\left(Q'\gamma, \I \left(\frac{\sigma^2 / m_n + d}{N_n-1}\right) \right)
\end{align*}
independently of each other and of the distribution of $q'\overline{Y}_j$ and $q'\overline{Y}_{-j}$.

We now consider a limit with $n \rightarrow \infty, m_n \rightarrow m, N_n \rightarrow \infty, K_n \rightarrow K$.
In this limit we learn $\alpha$ and $\gamma$ from $\overline{Y}_{-j}$,
and
\begin{align*}
    Q'\mu_j &| \overline{Y}_{j}, \gamma
    \sim \mathcal{N}\Bigg(\underbrace{Q'\gamma + \frac{d}{d+\sigma^2/m} Q'(\overline{Y}_{j} - \gamma)}_{=Q'\hat{m}_j}, \I \frac{d^2}{d+\sigma^2/m}\Bigg),
    \\
    q'\mu_j &| \overline{Y}_{j}, \alpha, \gamma
    \sim \mathcal{N}\Bigg(\underbrace{q'\1 \alpha + q'\gamma + \frac{b{+}d}{b{+}d+\sigma^2/m} (q'(\overline{Y}_{j}-\gamma) - \alpha)}_{=q'\hat{m}_j}, \frac{(b{+}d)^2}{b{+}d+\sigma^2/m}\Bigg),
    \\ 
    Q'\hat{m}_j
    &= Q'\gamma + \frac{d}{d+\sigma^2/m} \left(Q'\delta_j + \mathcal{N}\left(\0,\I \frac{\sigma^2}{m}\right)\right) \sim
    \mathcal{N}\left(\0,
    \I
        \left(
            c + \frac{d^2}{d+\sigma^2/m}
    \right)\right),
    \\
    q'\hat{m}_j &= q'\1 \alpha + q'\gamma + \frac{b{+}d}{b{+}d+\sigma^2/m} \left(q'\1 \beta_j + q'\delta_j + \mathcal{N}\left(0,\frac{\sigma^2}{m}\right)\right) \sim
    \mathcal{N}\left(\0, 
            a{+}c + \frac{(b{+}d)^2}{b{+}d+\sigma^2/m}
    \right).
\end{align*}

The optimal assignment policy that maximizes expected utility (and, in this homoscedastic case, also assignment probabilities) is
\begin{equation*}
    \hat{a}(j)
    =
    \argmax_{k} (Q \underbrace{\mathrm{E}\left[Q'\mu_j\middle|\overline{Y}_{j},\gamma\right]}_{=Q'\hat{m}_j})^k.
\end{equation*}
Give the data, the posterior expectation of average outcome and best-arm probability under $\hat{a}(j)$ are
\begin{align*}
    \mathrm{E}[\mu_j^{\hat{a}(j)}|\overline{Y}_{j},\alpha,\gamma]
    &= 
    \max_k \hat{m}_j^k = q'\hat{m}_j + (Q Q'\hat{m}_j)^{(1)}
    \\
    \mathrm{P}(\hat{a}(j) = \argmax_k{\mu_j^k}|\overline{Y}_{j},\alpha,\gamma)
    &= \mathrm{P}\left((Q Q'\hat{m}_j)^{(1)} = \left(Q \mathcal{N}\left(Q'\hat{m}_j,\I \frac{d^2}{d+\sigma^2/m}\right)\right)^{(1)}
    \middle| \hat{m}_j \right)
\end{align*}
where we write $x^{(1)}$ for the maximum of a vector.

We compare this assignment to the assignment $\underline{a}(j)$ that randomizes treatment arms equally and the infeasible optimal oracle assignment $a^*(j) = \argmax_k \mu_j^k$.
Averaging over the prior distribution and the data,
\begin{align*}
    \mathrm{E}[\mu_j^{\hat{a}(j)}]
    &=
    \mathrm{E}\left[\left(Q \mathcal{N}\left(\0,\I \left(c + \frac{d^2}{d+\sigma^2/m}\right)\right)\right)^{(1)}\right]
    = \sqrt{c + \frac{d^2}{d+\sigma^2/m}} \mathrm{E}\left[\left(Q \mathcal{N}\left(\0,\I\right)\right)^{(1)}\right],
    \\
    \mathrm{E}[\mu_j^{a^*(j)}]
    &= \mathrm{E}\left[\left(Q \mathcal{N}\left(\0,\I \left(c + d \right)\right)\right)^{(1)}\right]
    = \sqrt{c + d} \mathrm{E}\left[\left(Q \mathcal{N}\left(\0,\I\right)\right)^{(1)}\right],
    \\
    \mathrm{E}[\mu_j^{\underline{a}(j)}]
    &= 0,
\end{align*}
so
\begin{equation*}
    \frac{\mathrm{E}[\mu_j^{\hat{a}(j)}] - \mathrm{E}[\mu_j^{\underline{a}(j)}]}{\mathrm{E}[\mu_j^{a^*(j)}]- \mathrm{E}[\mu_j^{\underline{a}(j)}]}
    =
    \sqrt{\frac{c + \frac{d^2}{d+\sigma^2/m}}{c + d}}
    = \sqrt{1 - \frac{d^2 \sigma^2 }{ (m d + \sigma^2)(c+d)}},
\end{equation*}
invariant to $K$.
At the same time,
\begin{align*}
    \mathrm{P}(\hat{a}(j) = \argmax_k{\mu_j^k}|\overline{Y}_{j},\alpha,\gamma) 
    & \rightarrow 0
\end{align*}
in probability
as $K \rightarrow \infty$, while $\mathrm{P}(a^*(j) = \argmax_k{\mu_j^k}) = 1$ and $\mathrm{P}(\underline{a}(j) = \argmax_k{\mu_j^k}) = 1 / (K+1)$,
so
\begin{equation*}
    \frac{\mathrm{P}(\hat{a}(j) = \argmax_k{\mu_j^k}) - \mathrm{P}(\underline{a}(j) = \argmax_k{\mu_j^k})}{\mathrm{P}(a^*(j) = \argmax_k{\mu_j^k}) - \mathrm{P}(\underline{a}(j) = \argmax_k{\mu_j^k})}
    \rightarrow 0
\end{equation*}
as $K \rightarrow \infty$.

Consider now the alternative assignment rule
\begin{equation*}
    \tilde{a}(j) = \argmax_k \mathrm{E}[\mu^k_j|\overline{Y}_j^k,\overline{Y}_{-j}^k]
\end{equation*}
that estimates arms separately to minimize individual mean-squared error.
In the limiting regime with $n \rightarrow \infty, m_n \rightarrow m, N_n \rightarrow \infty, K_n \rightarrow K$,
$\mathrm{E}[\overline{Y}_{-j}^k|\alpha,\gamma] = \alpha + \gamma^k$ is known, and
\begin{equation*}
    \mathrm{E}[\mu^k_j|\overline{Y}_{j}^k,\alpha + \gamma^k]
    =
    \alpha + \gamma^k + \frac{b+d}{b+d+\frac{\sigma^2}{m}} (\overline{Y}_{j}^k - \alpha - \gamma^k).
\end{equation*}
Hence,
\begin{align*}
    \tilde{a}(j) &= \argmax_k \Bigg(Q \underbrace{\left(Q' \gamma + \frac{b+d}{b+d+\frac{\sigma^2}{m}} (Q'\overline{Y}_{j} - Q'\gamma)\right)}_{=Q'\tilde{m}_j}\Bigg)^k
\end{align*}
Relative to the optimal rule, this rule is equivalent to a rule that incorrectly attributes variation in the baseline $\beta_j$ (which does not affect the ranking) to relative variation in arm-wise means (which would affect the ranking), therefore overfitting to the data relative to the true posterior mean $Q'\hat{m}_j = Q' \gamma + \frac{d}{d+\frac{\sigma^2}{m}} (Q'\overline{Y}_{j} - Q'\gamma)$ of $Q'\mu_j$.
To compare performance, we invoke the following result that provides a generalization of the above calculus around maximizers of Normal random variables:

\begin{lemma*}
    Assume that $X, Y$ jointly multivariate Normal of the same dimension with mean zero and $\Var(X) = \I x, \Var(Y) = \I y, \Cov(X,Y) = \I z$.
    Then
    $
        \mathrm{E}[Y^{\argmax_k X^k}]
        = \frac{z}{x} \mathrm{E}[X^{(1)}]
        = \frac{z}{\sqrt{x}} \mathrm{E}[\mathcal{N}(\0,\I)^{(1)}]
        = \frac{z}{\sqrt{x \: y}} \mathrm{E}[Y^{(1)}].
    $
\end{lemma*}
\begin{proof}
    Note that
    $Y =  X \frac{z}{x} + Y -  X \frac{z}{x} = (X + Z)\frac{z}{x}$
    where $Z = Y \frac{x}{z}  - X$
    fulfills $\Cov(Z,X) = \I (\frac{x}{z} z - x) = \O$,
    so
    $
        \mathrm{E}[X^{\argmax_k Y}]
        = \frac{z}{x} \mathrm{E}[(X+Z)^{\argmax_k X^k}] \:
        = \frac{z}{x} \mathrm{E}[X^{(1)}]
        = \frac{z}{\sqrt{x}} \mathrm{E}[\mathcal{N}(\0,\I x / x)^{(1)}].
    $
\end{proof}

Since only relative rankings of the vectors matter, we can apply the lemma e.g. to conformal $X,Y$ with $Q'X = Q'\tilde{m}_j, Q'Y = Q'\mu_j$ with diagonal variances and covariances
\begin{align*}
    \Var(X) &= \I \left(c + \frac{(b+d)^2\left(d + \frac{\sigma^2}{m}\right)}{\left(b + d + \frac{\sigma^2}{m}\right)^2}\right),
    &
    \Cov(X,Y) &= \I \left(c + \frac{d (b+d) }{b + d + \frac{\sigma^2}{m}}\right)
\end{align*}
to find that
\begin{equation*}
    \mathrm{E}[\mu_j^{\tilde{a}(j)}]
    =
    \frac{c \left(b + d + \frac{\sigma^2}{m}\right) + d (b+d)}{\sqrt{c \left(b + d + \frac{\sigma^2}{m}\right)^2 +(b+d)^2\left(d + \frac{\sigma^2}{m}\right)}}
    \mathrm{E}\left[\left(Q \mathcal{N}\left(\0,\I\right)\right)^{(1)}\right],
\end{equation*}
where we have used that $\mathrm{E}\left[\left(Q \mathcal{N}\left(\0_K,\I_K\right)\right)^{(1)}\right] = \mathrm{E}\left[\left(\mathcal{N}\left(\0_{K+1},\I_{K+1}\right)\right)^{(1)}\right]$.
In particular,
\begin{equation*}
    \frac{\mathrm{E}[\mu_j^{\tilde{a}(j)}] - \mathrm{E}[\mu_j^{\underline{a}(j)}]}{\mathrm{E}[\mu_j^{\hat{a}(j)}]- \mathrm{E}[\mu_j^{\underline{a}(j)}]}
    =
    \sqrt{
    \frac{\left(c \left(b + d + \frac{\sigma^2}{m}\right) + d (b+d)\right)^2\left(d + \frac{\sigma^2}{m}\right) }{\left(c \left(b + d + \frac{\sigma^2}{m}\right)^2 +(b+d)^2\left(d + \frac{\sigma^2}{m}\right)\right) \left( c\left(d + \frac{\sigma^2}{m}\right) + d^2 \right)}
    },
\end{equation*}
which is one for $b=0$ or $c=0$ (in which cases choices are the same) and smaller than one otherwise.

\end{document}